\documentclass{article}
\usepackage[cmex10]{amsmath}
\usepackage{algorithmicx}

\usepackage{fleqn}    
\usepackage{booktabs} 
\usepackage{multirow} 

\usepackage{subfig}
\usepackage{subfloat}

\usepackage{color}
\usepackage{ifthen}
\usepackage{eucal}

\usepackage{graphicx}

\usepackage{amssymb}
\usepackage{amsbsy}

\usepackage{algorithm}
\usepackage{algorithmicx}
\usepackage{algpseudocode} 

\usepackage{amsthm}
\newtheorem{theorem}{Theorem}

\usepackage{url}

\usepackage{bm}

\usepackage{pdfpages}

\usepackage{titletoc}

       \newcommand{\vah}{\hat{\bm{a}}}        \newcommand{\ah}{\hat{a}}

\newcommand{\vw}{\bm{w}}               
\newcommand{\vx}{\bm{x}}

\newcommand{\vphi}    {\bm{\phi}}             
\newcommand{\vvarphi}    {\bm{\varphi}}

\newcommand{\vtheta}  {\bm{\theta}}

\newcommand{\vpsi}    {\bm{\psi}}

    \newcommand{\Ac}{\mathcal{A}}

    \newcommand{\Xc}{\mathcal{X}}

\newcommand{\R}{\mathbb{R}}

\DeclareMathOperator*{\argsortk}{arg\,sort^k}
\DeclareMathOperator*{\sortk}{sort^k}

\newcommand{\fig}[1]{Figure~\protect\ref{#1}}

\newcommand{\Sec}[1]{Section~\protect\ref{#1}}

\newcommand{\eq}[1]{(\protect\ref{#1})}

\newcommand{\indicator}{\mathbb{I}}
\newcommand{\I}{\mathbb{I}}
\renewcommand{\P}{\mathbb{P}}
\newcommand{\E}{\mathbb{E}}

\begin{document}
\title{Context-Based Prediction of App Usage}
\author{Joseph Keshet\footnote{J. Keshet is with the Department
of Computer Science, Bar-Ilan University, Ramat-Gan, Israel, 52900.}, Adam Kariv\footnote{A. Kariv, A.Dagan, D. Volk and J. Simhon are with EverythingMe and Doat Media Ltd., Tel-Aviv, Israel, 64921.}, Arnon Dagan,$^\dagger$ Dvir Volk,$^\dagger$ and Joey Simhon$^\dagger$}
\maketitle
\begin{abstract}
There are around a hundred installed apps on an average smartphone. The high number of apps and the limited number of app icons that can be displayed on the device's screen requires a new paradigm to address their visibility to the user. In this paper we propose a new online algorithm for dynamically predicting a set of apps that the user is likely to use. The algorithm runs on the user's device and constantly learns the user's habits at a given time, location, and device state. It is designed to actively help the user to navigate to the desired app as well as to provide a personalized feeling, and hence is aimed at maximizing the AUC. We show both theoretically and empirically that the algorithm maximizes the AUC, and yields good results on a set of 1,000 devices.
\end{abstract}

\section{Introduction}
\label{sec:introduction}

%
%
%
%

Smartphones are one of the most widely used devices nowadays. It is estimated that today about half the adult population owns a smartphone. The average American actively uses a smartphone more than two hours every day, and nearly 80\% of smartphone-owners check messages, news or other services within 15 minutes of getting up. 

One of the things that made smartphones ubiquitous is their ability to execute countless apps. Those apps take the advantage of the device's high computation power, constant internet connection, and features like location services, and use them to handle various tasks. Google Play store offers 1.5 million apps for Android users. On the average, there are 97 installed apps on typical smartphone according to the logs of \emph{EverythingMe}, and similarly, there are 96 installed apps according to \emph{Yahoo Aviate}'s logs \cite{baeza2015predicting}. The high number of installed apps and the limited number of app icons that can be displayed on the device's screen, requires a new paradigm to address their visibility to the user. 

In this paper we propose a new algorithm for dynamically predicting a set of apps that the user is likely to use. A set of app icons (usually four) are presented on a special dock, called \emph{Prediction Bar}, and dynamically change according to the user's habits at a given time, location, and device state (headphone connected, bluetooth active, and so on). For example, a user often uses the app \emph{Evernote} to take notes during the morning at her office, but never at home or during the weekend. Every Saturday morning she goes to the farmers market and then uses \emph{BestParking} to help her find a nearby parking place. Every now and then she uses \emph{Facebook}. The time, location and the device state are all considered as the \emph{context} of the user's device. The goal of the algorithm is, given the context information, to actively help the user navigate to the desired app as well as to provide a personalized feeling. 

The algorithm proposed here is based on learning the user's preferences in an online fashion. The algorithm keeps a weight vector (a set of parameters) for each of the installed apps. The user's contextual information is represented as a vector of real numbers called \emph{feature primitive} that is mapped to a high dimensional abstract space, to have a more meaningful representation for the learning process. The projection of the mapped feature primitives onto any of the app's weight vector is the score associated with that app. At a given context the algorithm selects the set of apps that attain the highest score, and presents them to the user as icons in the Prediction Bar. Then the user clicks on one of the apps, either in the Prediction Bar, or elsewhere. If the clicked app is not presented in the Prediction Bar, or is in the Prediction Bar, but with a low score, then the algorithm updates the weight vector associated with the clicked app. The algorithm also updates all the weight vectors of the apps that where mistakenly predicted with a higher score than the clicked app.

One way to assess the performance of the algorithm is by checking the precision of the prediction, that is, the number of times the user clicks on the apps displayed in the Prediction Bar relative to the number of times she clicks on apps displayed anywhere (including the Prediction Bar). While this might sound a very reasonable measure of performance, it seems that this evaluation metric tends to predict the most frequent apps and ignores rarely used apps. It prefers, for example, the prediction of frequently used apps, like \emph{Facebook}, at \emph{any} time and location, over the prediction of apps like \emph{GateGuru}, which the user uses only at airports, or \emph{BestParking}, which the user uses every Saturday morning when she searches for a parking spot. An empirical study conducted by \emph{EverythingMe} showed that the naive paradigm that constantly presents the most frequent apps in the Prediction Bar, attains the highest precision (see \Sec{sec:auto_prediction}). This, of course, does not serve our goals of helping the user navigate to the desired app or of providing a personalized feeling. 

As we shall see from the empirical analysis, we prefer to assess the quality of the prediction using a different metric, namely, the area under the receiver operating characteristic curve (AUC). This measure of performance tends to prefer predictions with a high true positive rate and a low false positive rate, which means that a mis-prediction of any app has the same cost -- no matter how frequent the app is.

We propose an efficient online algorithm that is executed on the device. It is based on the Passive-Aggressive online algorithmic framework \cite{crammer2006online}, adapted to maximize the AUC at each round. While there exists algorithms to maximize the AUC, such as \cite{joachims2005support,rosenfeld2014learning}, they are different from the algorithm proposed here in several aspects. Firstly, the algorithms \cite{joachims2005support,rosenfeld2014learning} are \emph{batch} algorithms. In the batch setting the input to the algorithm is a training set of labeled examples, and the output of the algorithm is a hypothesis that should perform well on an unseen data that drawn 
from the same distribution of the training set. In the \emph{online} setting, which we are interested in this paper, the algorithm constantly adapts the hypothesis, and there is partition of the data into a training set or a test set. The online learning algorithms works in rounds, where at each round the algorithm get as an input a feature vector that represents the device context and has to predict the next apps the user will use. Based on the feedback from the user, the algorithm updates its hypothesis. The online algorithm does not need to be evaluated on a test set of unseen data, but on the next behavior of the user. Hence online algorithm inherently supports a drifting hypothesis, that is, it support a drift or a change of the user's preferences over time. Secondly, the algorithms \cite{joachims2005support,rosenfeld2014learning} are based on structural support vector machine (SSVM) \cite{TsochantaridisHoJoAl05}. They assume that the training set is given as pairs of a feature vector and a binary label. At every round, they need a loss-augment inference with the AUC loss, which is computationally heavier than a simple inference. Our algorithm, on the other hand, assumes that each example is composed of pairs of two feature vectors: one that represents a context in which the app is used and the other represents a context where the app is not used. This might be more efficient than the used of loss-augmented inference, and leads to a strongly consistent preditor when converted to a batch algorithm. Since we are interested in the online setting in the paper, we defer the theoretical comparison to a different paper.

A important issue in practical implementation of the algorithm is how to initiate the Prediction Bar for new users, when we do not have any of their app preferences. We will discussion on this topic in the \Sec{sec:learning}. The interested reader can find more ideas in \cite{baeza2015predicting}. 

The contribution of this paper is the following: (i) a new machine learning online algorithm for ranking the set of installed apps so as to maximized the AUC; (ii) theoretical analysis of the algorithm; (iii) a large scale empirical study on the performance of the algorithm on a real users' data. 

The paper is organized as follows. In \Sec{sec:related_work} we present previous work on predicting app usage. In \Sec{sec:problem_setting} we formally state the notation and the problem definition. Then in \Sec{sec:auto_prediction} we describe the motivation of using the AUC as an evaluation metric. We continue with a detailed derivation of the algorithm and present its theoretical analysis in \Sec{sec:learning}. The features used in the prediction are discussed in \Sec{sec:features}. Experimental results are presented in \Sec{sec:experiments}. We conclude the paper with a discussion in \Sec{sec:discussion}.

\section{Related Work}\label{sec:related_work}

The problem of predicting the app the user is about to use has been recently addressed by many research groups. The work in \cite{Huang:2012:PMA:2370216.2370442} was one of the first. The authors proposed an app predictor based Bayesian Networks and contextual information such as time, location, and the user profile. They evaluated their results using average prediction rate on the IDIAP/Nokia MDC dataset \cite{laurila2012mobile}, which contains a small group of 38 users. Similarly, in \cite{zou2013prophet} the author proposed a light-weighted Bayesian methods to predict the next app based on the app usage history. They also evaluated their result on the MDC dataset. Yahoo Aviate team used Bayesian Networks as a learning algorithm and evaluated it on a larger set of 480 active users, and compare their system to other standard learning algorithm \cite{baeza2015predicting}. They are one of the few groups who study the cold-start problem, when no data is available on the user. In \cite{Huai:2014:TPC:2672614.2629504} the temporal user's behavior was also taken into account by using an HMM-based sequence prediction. 

Most works are based on contextual features related to the time, location, phone state and environment. Some authors \cite{yan2011appjoy,liao2012mining}, however, proposed to identify the usage patterns and user rating, without taking into account usage context. Interestingly, in \cite{liao2012mining} the authors propose to detect the periodicity patterns of usage by the Fourier transform, and then scoring them by  Chebyshev's inequality. 

Several authors based their prediction on similar users or a general group of users. Both \cite{yan2011appjoy,natarajan2013app} are based on the adaptation of a collaborative filtering algorithm as an app prediction algorithm. The work presented in \cite{xu2013preference} leverages the user-specific models by patterns of community app behaviors, guided by user similarities. 

Related works also include predictions of users behavior with mobile devices. See, for example,  \cite{verkasalo2009contextual,fitchett2012accessrank,etter2013go,do2014and,lin2014mining}, and the many refereces therein. Another set of works concerns a smart caching mechanism for fast app pre-loading  \cite{yan2012fast,parate2013practical}. 

All the works above have been focused on the maximizing the average precision or the prediction rate. In that sense our work is unique, as it proposes a new theoretically founded algorithm for the task of app prediction. Most of the works has been focused on generative models (Bayesian Networks and HMM) and our work is based on discriminative models (large margin and kernel methods). We would like to note that we found it difficult to directly compare our algorithm to previous work. The main reason is that there is no single benchmark for this tasks, and the few dataset which exists are either restricted to non-profit organizations (IDIAP/Nokia MDC dataset) or not available online (AppJoy dataset).

\section{Problem Setting}
\label{sec:problem_setting}

In this section we set the notation and formally define the problem of online prediction of app usage. Our goal is to predict the most probable apps the user will use given her location, activity, time, device status, and so on. The input to the system, therefore, is a \emph{feature primitive}, $\vx\in\Xc$, that represents the user's contextual information as a vector of $n$ real numbers, where $\Xc\subset\R^n$ is the domain of possible contextual information. The concrete representation and description of the features are discussed in \Sec{sec:features}.

We denote the set of user's installed apps by $\Ac$, and their number by $K=|\Ac|$. The Prediction Bar presents a set of $k$ apps (currently $k$=4), and it is assumed that at a given context the user may click on one of the $k$ apps ($k\le K$) with a high probability. Our goal is to find a function that outputs the set of $k$ apps the user is about to click on, given the input context. The predicted apps should reflect their relevance according to the user's preference at the given context.

We assume that there is a score function, $f: \Xc \times \Ac \to \R$, that assigns a score to every app $a\in\Ac$ given the user's context $\vx\in\Xc$. High score of $f(\vx, a)$ means that app $a$ is relevant at context $\vx$. We define the set of all scores as $F(\vx, \Ac) = \{ f(\vx,a) \, |\, a\in\Ac\}$. Define the $k$-\emph{sort} function as a function that returns a vector of the top $k$ ordered scores, given the set of all scores, namely, $\sortk : F(\vx,\Ac) \to \R^k$. The argument of the $k$-\emph{sort} function is the set of apps corresponding to the top $k$ ordered scores. The prediction of the set of $k$ best apps is given as
\begin{equation}
\vah = \argsortk F(\vx,\Ac),
\end{equation}
where $\vah\in\Ac^k$ is a vector of $k$ apps, ordered such that the most relevant app is first, the second most relevant is second, and so on.

The performance of the prediction is measured by a cost function $\gamma(a, \vah)$ between the clicked app $a\in\Ac$ and a set of $k$ predicted apps $\vah$, which checks the existence of the clicked app $a$ within the set $\vah$ and returns a positive real number, namely, $\gamma: \Ac \times \Ac^k \to \R_+$. For example, the 0-1 cost function is defined to be 0 if $a\in\vah$ and  1 otherwise, that is,
\begin{equation}\label{eq:0-1cost}
\gamma(a, \vah) = \indicator[a \notin \vah],
\end{equation}
where $\I[\pi]$ is an indicator function and it equals to 1 if the predicate $\pi$ holds and 0 otherwise. If $k=1$ this cost function reduces to the standard 0-1 binary loss function. In \Sec{sec:auto_prediction} we show that using the 0-1 cost function as a measure of performance favors the set of $k$ most frequently used apps over other types of predictions, and we then present the reason to introduce other cost functions.

We propose an online learning algorithm for predicting a set of $k$ apps. The online algorithm maintains a set of parameters. We denote the parameters the $t$-th round by $\vtheta_t$, and indicating that by adding a subscript to the prediction function $f_{\vtheta_t}$. Each round corresponds to a click on an app. After each such click the algorithm updates it parameters and generates a new predictor $f_{\vtheta_{t+1}}$. Our goal is to estimate the parameters so as to maximize the cumulative AUC along its run.
\section{Automatic App Usage Prediction}
\label{sec:auto_prediction}

\begin{figure*}[th]
\centering
\includegraphics[height=8cm]{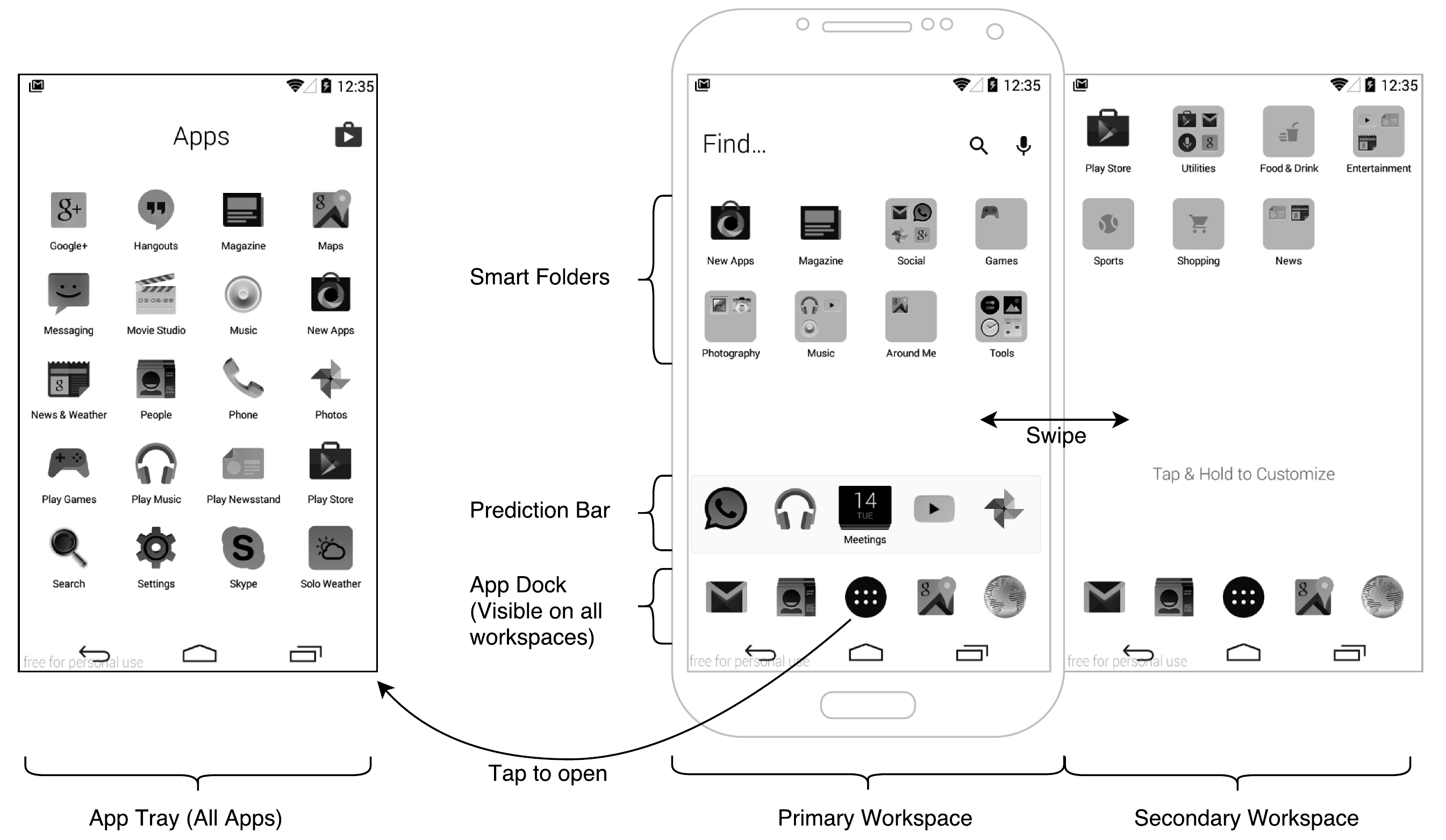}
\caption{Screenshot of the mobile device.}
\label{fig:screenshot}
\end{figure*}

We briefly review the device screen organization. A schematic screenshot of the device is depicted in \fig{fig:screenshot}. In this particular device there are two Workspaces. A Workspace is a virtual screen where the user can place app icons and app folders. The user can change Workspaces by swiping the screen left or right. The primary Workspace is also called the Home Screen. At the bottom of all Workspaces there is a set of 5 icons which is called the App Dock. The App Dock is composed of 4 user defined app icons, while the middle icon opens the App Tray - a folder which contains all of the installed apps. On the first Workspace, above the App Dock, there is another set of 5 icons which are called the Prediction Bar. Out of those 5 icons, 4 represent apps that are automatically predicted according to the context information.

Apps that are located on the Home Screen or on the App Dock are considered as highly available apps and they should not be considered to be predicted in the Prediction Bar. By presenting to the user apps which are already located in the same screen, we create un-necessary duplication and offer no value at all. This characteristic has an interesting implication. These highly available app clicks we avoid also tend to be the ones that are most used by the user. About 35\% of app clicks are done on application shortcuts which are placed on the Home Screen or the App Dock; these clicks are distributed among only 4.6 apps on average. The rest of the clicks, i.e., the ones which occur in other locations, are distributed among 17.6 apps on average. This means that our prediction algorithm needs to focus on the 2nd-tier apps, and not on the mostly used apps. Hence, from now on we will restrict ourselves to the prediction of those apps which are not on the Home Screen.

In order to generate the automatic prediction, the device keeps track of all app clicks. For each such event the device records the time of the event, the device location (both exact coordinates as well as an indication of a recurring location), and hardware related signals (e.g., headphones connected indication, current Wi-Fi network, Bluetooth devices, etc.). Along with these fields, the position of the app on the device's screen is also recorded. Specifically we record both the relative location of the app's icon on the display (e.g., first row and third column of the secondary workspace), and if it was clicked on a special location (e.g. App Tray, App Dock, Prediction Bar, a search result, a shortcut inside a folder, etc.) --- all those form the user's context feature vector $\vx$.

One of the most trivial ways to build the Prediction Bar is to constantly predict the $k$ most frequent apps. We name this method \emph{the $k$ Most Frequently Used}, or $k$MFU. Formally, we define the score function for an app $a$ as follows:
\begin{equation}
f^\mathrm{kMFU}(\vx,a) = \left\{ \begin{array}{ll}
1 & \textrm{if ~ $a \in \Ac_{k{\rm MFU}}$}\\
0 & \textrm{otherwise}
  \end{array} \right. ,
\end{equation}
where $\Ac_{k{\rm MFU}}$ is the set of $k$ most used installed apps.

A bit more dynamic method combines the most frequent apps and the most recent apps and is called the \emph{Frecency} algorithm. This algorithm was originally used for cache management \cite{lee1999existence}, and adapted to an algorithm for predicting users' behavior \cite{fitchett2012accessrank}. Denote by $t^a_i$ the time of the $i$-th click on app $a$, where there were overall $n^a$ clicks on this app in the last $T$ days. Then the score function of an app $a$ is defined as follows:
\begin{equation}\label{eq:frecency}
f^{\mathrm{Frecency}}(\vx,a) = \sum_{i=1}^{n^a} {p}^{\frac{t-t^a_i}{T}},
\end{equation}
where $p$ and $T$ are fixed parameters (in our setting $p=0.1$ and $T=60$ days performed best).

Before we considered designining a new online learning algorithm, we tested the $k$MFU and the Frecency algorithms on our data as follows (full details in \Sec{sec:experiments}). We  chose uniformly at random a set of 1000 devices that contains a total of 5,181,312 app clicks. For each device we executed both $k$MFU and Frecency and tried to predict a set of $k=4$ apps. We checked the average 0-1 cost as in \eq{eq:0-1cost} of each algorithm on each device and averaged the results on all devices. This is the average precision of the algorithm, averaged on the rounds and the devices.

We found out that both $k$MFU and Frecency algorithms had very high precision (see \fig{fig:compare-acc}). This, however, does not serve our goal: the $k$MFU algorithm simply puts the same 4 apps in the Prediction Bar, and does not predict the next app or gives a personalized feeling. The Frecency algorithm tends to predict frequent used apps and ``forgets'', for example, bi-weekly or weekly app click patterns.

Measuring the performance of the algorithm using \emph{precision} is biased toward the frequency of the app usage. In our case a better evaluation metric would be the area under the receiver operating characteristic curve (AUC) \cite{fawcett2006introduction}. This measure of performance tends to prefer predictions with a high true positive rate and a low false positive rate, which means that a mis-prediction of any app has the same cost no matter how frequent the app is. In the next section we derive an online algorithm that predict the next app and aims at maximizing the AUC metric.
\section{Learning Apparatus}
\label{sec:learning}

We saw in the previous section that the the usage of 0-1 cost function, $\gamma(a, \vah) = \indicator[a \notin \vah]$, to assess the performance favors the set of $k$ most used apps over other predictions. This lead to a static prediction, which does not reflect the user's instantaneous context-based app usage. We turn to describe a different measure of evaluation.

The receiver operating characteristic curve (ROC) is a graphical plot of the true positive rate  ($a\in\vah$) as a function of the false positive rate ($a\notin\vah$). The points on the curve are obtained by sweeping the decision threshold from the most positive confidence value to the most negative one. Hence, the choice of the threshold represents a trade-off between different operational settings, corresponding to cost functions weighting false positive and false negative errors differently. Assuming a flat prior over all cost functions, it is appropriate to select the system maximizing the averaged performance over all settings, which corresponds to the model maximizing the area under the ROC curve (AUC). In the following we describe a large margin approach which aims at predicting a set of apps which achieves a high AUC.

Let us denote by $\Xc^{a,+}$  the set of contexts (location, times, WiFi networks, etc.) where  app $a$ is clicked. Similarly denote by $\Xc^{a,-}$ the set of contexts where  app $a$ is never clicked. The AUC can be defined as \cite{bamber1975area,hanley1982meaning}
\begin{equation}\label{eq:auc_def}
\text{AUC}=\P[f_{\vtheta}(\vx^{a,+},a) >  f_{\vtheta}(\vx^{a,-}, a)],
\end{equation}
where $\vx^{a,+}\in \Xc^{a,+}$ and $\vx^{a,-}\in \Xc^{a,-}$. That is, the AUC is the probability that the confidence for app $a$ in a context that it is clicked is higher than the confidence where it is not clicked. The probability is over a triplet $(\vx^{a,+}, \vx^{a,-}, a)$ drawn from a fixed unknown distribution $\rho$. Our goal is find the parameters $\vtheta$ that maximizes the AUC, namely
\begin{equation}
\vtheta^* = \arg \max_{\vtheta} ~ \P[f_{\vtheta}(\vx^{a,+},a) >  f_{\vtheta}(\vx^{a,-}, a)],
\end{equation}
or equivalently,
\begin{equation}
\vtheta^* = \arg \max_{\vtheta} ~ \E\left[\I\left[f_{\vtheta}(\vx^{a,+},a) >  f_{\vtheta}(\vx^{a,-}, a)\right]\right],
\end{equation}
where the expectation is with respect to the distribution $\rho$. Since the distribution $\rho$ is unknown, this optimization problem cannot be solved directly. For methodological reasons we will first discuss the batch setting, where a training set of $m$ examples are sampled from the distribution $\rho$. We replace the expectation with the sample average and add a regularization factor to avoid overfitting. The objective is now
\begin{equation}\label{eq:objective}
\min_{\vtheta} ~ \frac{1}{m}\sum_{i=1}^{m}\I\left[f_{\vtheta}(\vx^{a_i,+}_i,a_i) <  f_{\vtheta}(\vx^{a_i,-}_i, a_i)\right] + \lambda \Omega(\vtheta),
\end{equation}
where $\lambda$ is a trade-off parameter between the loss term and the regularization. Conventionally, we replaced the $\max$ operation with a $\min$ operation while changing the direction of the inequality sign. Still, we cannot optimize this objective since the summands are indicator functions, a combinatorial quantity that is hard to be minimized directly. Different learning algorithms use various surrogate loss function to approximate this function. Here we focused on large margin based algorithms, which replaces it with a convex upper bound surrogate loss called \emph{hinge loss} we describe in the next subsection.

\subsection{A Large Margin Algorithm for App Ranking}

Building on techniques used for learning multiclass \cite{crammer2002algorithmic} and structured prediction classifiers \cite{TaskarGuKo03,TsochantaridisHoJoAl05}, our set of prediction functions distills to a classifier in this vector-space which is aimed at separating the relevant apps from irrelevant ones. We focus on the following set of linear prediction functions: 
\begin{equation}\label{eq:linear_score_fucntion}
 f_{\vtheta}(\vx, a) = \vtheta\cdot\vvarphi(\vx,a),
\end{equation}
where $\vtheta\in \R^d$ are the model parameters , and $\vvarphi: \Xc \times \Ac \to \R^d$ is a set of functions called \emph{feature functions} or \emph{feature maps}. Each feature function gets as input a feature primitive $\vx$ and an app $a$ and returns a scalar which, intuitively, should be correlated with whether the app $a$ is used in context of $\vx$. For example, one such function might be how many times the app $a$ was clicked in a radius of 50 meters around the location \emph{xyz}. Refer to \Sec{sec:features} for mode details.
 
Methods like support vector machines (SVMs) replace the summands of \eqref{eq:objective} with their corresponding convex upper bound called \emph{hinge loss}. Denote $[\pi]_+ = \max\{\pi,0\}$, then in our case:
\begin{align}
\I\Big[\vtheta&\cdot\vvarphi(\vx^{a_i,+}_i,a_i) < \vtheta\cdot\vvarphi(\vx^{a_i,-}_i, a_i)\Big] \\
 &\le \Big[ 1 - \vtheta\cdot\vvarphi(\vx^{a_i,+}_i,a_i) +  \vtheta\cdot\vvarphi(\vx^{a_i,-}_i, a_i)  \Big]_+ \\
 &\le \Big[ 1 - \vtheta\cdot\vvarphi(\vx^{a_i,+}_i,a_i) + \max_a ~ \vtheta\cdot\vvarphi(\vx^{a_i,-}_i, a)  \Big]_+.
 \end{align}
We define the hinge surrogate loss as follows:
\begin{equation}\label{eq:auc_loss}
\ell(\vx^{a_i,+}_i,\vx^{a_i,-}_i,a_i, \vtheta) = \Big[ 1 - \vtheta\cdot\vvarphi(\vx^{a_i,+}_i,a_i) 
+ \max_a ~ \vtheta\cdot\vvarphi(\vx^{a_i,-}_i, a) \Big]_+.
 \end{equation}
Using this surrogate loss instead of the cost in the objective \eqref{eq:objective}, and using the $\ell_2$ regularization, we get the structural SVM algorithm for maximizing the AUC in our setting
\begin{equation}
\min_{\vtheta}  ~ 
 \frac{1}{m} \sum_{i=1}^{m}\ell(\vx^{a_i,+}_i,\vx^{a_i,-}_i,a_i, \vtheta) ~+~ \frac{\lambda}{2}\|\vtheta\|^2.
\end{equation}
This is a convex function in its parameters and its solution can be found using the cutting plane method \cite{TsochantaridisHoJoAl05} or by stochastic sub-gradient descent \cite{shalev2011pegasos}. In out setting we are interested in optimizing the set of parameters for each user separately on the device. Moreover we would like to handle scenarios where the users preferences changes over time. Hence we turn to online learning algorithm.  

\subsection{An On-line Algorithm}

We now describe an online algorithm for learning the parameters, while maximizing the AUC instantaneously. It is a variant of the Passive-Aggressive algorithm  \cite{crammer2006online} for maximizing the AUC which we call \emph{AUC-PA}.

The online algorithm works in rounds, and updates its parameters every round. Set the initial parameters to $\vtheta_0=\boldsymbol{0}$. While the model extends prediction for every context, a learning round is defined by the event of app click. At the $t$-th round, when the device is at context $\vx_t$, the algorithm predicts a set of $k$ apps $\vah_t$ using the set of parameters $\vtheta_{t-1}$
\begin{equation}
\vah_t = \argsortk ~ \{ \vtheta_{t-1}\cdot\vvarphi(\vx_t,a) ~|~ a\in\Ac \}.
\end{equation}
The set of predicted apps $\vah^k_t$ are presented to the user. The user clicks on an app $a_t$ either from the set of $k$ presented apps or from the whole set of the $n$ installed apps $\Ac$. 

For each app $a$ we keep a set $\Xc^{a,-}$ of randomly chosen context, where the app $a$ is never clicked by the user. Once we know the user clicked the app $a_t$, we extract uniformly at random $\vx^{a_t,-} \in \Xc^{a_t,-}$. We then predict the most relevant app given the new feature primitive
\begin{equation}
\ah^-_t = \arg\max_{a\in\Ac} ~ \vtheta_{t-1}\cdot\vvarphi(\vx^{a_t,-},a).
\end{equation}
The model then suffers a loss $\ell(\vx^{a_t,+}_t,\vx^{a_t,-}_t,a_t, \vtheta_{t-1})$, defined in \eq{eq:auc_loss} or equivalently
\begin{equation}
\ell(\vx^{a_i,+}_i,\vx^{a_i,-}_i,a_i, \vtheta_{t-1}) 
=  1 - \vtheta_{t-1}\cdot\vvarphi(\vx^{a_i,+}_i,a_i) 
+  \vtheta_{t-1}\cdot\vvarphi(\vx^{a_i,-}_i, \ah^-_t),
\end{equation}
and updates the parameters as follows
\begin{equation}
\vtheta_t = \arg\min_{\theta} ~ \ell(\vx^{a_t,+}_t,\vx^{a_t,-}_t,a_t, \vtheta) ~+~ 
\frac{\lambda}{2} \| \vtheta  - \vtheta_{t-1} \|^2.
\end{equation}
The solution of this optimization problem is \cite{crammer2006online} 
\begin{equation}\label{eq:pa_update}
\vtheta_t = \vtheta_{t-1} + \tau_t \Big[ \vvarphi(\vx^{a_t,+},a_t) - \vvarphi(\vx^{a_t,-},\ah^-_t) \Big],
\end{equation}
where
\begin{equation}\label{eq:tau}
\tau_t = \min \left\{ \frac{1}{\lambda}, \frac{\ell(\vx^{a_t,+}_t,\vx^{a_t,-}_t,a_t, \vtheta_{t-1})}{\| \vvarphi(\vx^{a_t,+},a_t) - \vvarphi(\vx^{a_t,-},\ah^-_t) \|^2} \right\}.
\end{equation}
The overall algorithm is described in Algorithm~\ref{alg:auc_pa}.

\begin{algorithm}
\caption{The AUC-PA algorithm}\label{alg:auc_pa}
\begin{algorithmic}[1]
\State {\bf input:} parameter $\lambda$
\State {\bf init:} $\vtheta_0 = \boldsymbol{0}$
\For{$t=1,2,\dots$}
	\State new user's context $\vx_t$
  \State predict $\vah_t = \argsortk ~ \{ \vtheta_{t-1}\cdot\vvarphi(\vx_t,a) ~|~ a\in\Ac \}$
  \State the user clicked on $a_t$
	\State infer $\ah^-_t = \arg\max_{a\in\Ac} ~ \vtheta_{t-1}\cdot\vvarphi(\vx^{a_t,-},a)$
	\State suffer loss $\ell(\vx^{a_i,+}_i,\vx^{a_i,-}_i,a_i, \vtheta_{t-1})$
	\State set $\tau_t = \min \left\{ \frac{1}{\lambda}, \frac{\ell(\vx^{a_t,+}_t,\vx^{a_t,-}_t,a_t, \vtheta_{t-1})}{\| \vvarphi(\vx^{a_t,+},a_t) - \vvarphi(\vx^{a_t,-},\ah^-_t) \|^2} \right\}$
	\State update: $\vtheta_t = \vtheta_{t-1} + \tau_t \Big[ \vvarphi(\vx^{a_t,+},a_t) - \vvarphi(\vx^{a_t,-},\ah^-_t) \Big]$
\EndFor
\end{algorithmic}
\end{algorithm}

\subsection{Analysis}

We show that our online algorithm attains a high cumulative AUC after $T$ rounds, defined as follows 
\begin{equation}\label{eq:cumulative_auc}
\widehat{\text{AUC}}=\frac{1}{T}\sum_{t=1}^{T} \I\Big[\vtheta_t\cdot\vvarphi(\vx^{a_t,+}_t,a_t) > \vtheta_t\cdot\vvarphi(\vx^{a_t,-}_t, a_t) \Big],
\end{equation}
where $\vtheta_1,\ldots,\vtheta_T$ are the weight vectors obtained by the algorithm. The examination of the cumulative AUC is of great interest as it provides an estimator for the generalization performance. Note that at each round the algorithm can be considered as receives new example $(\vx^{a_t,+}_t, \vx^{a_t,-}_t, a_t)$ and predicts an app that is best suitable to $\vx^{a_t,-}_t$ using the previous weight vector $\vtheta_{t-1}$. Only after the prediction $\ah^-_t$ is made the algorithm suffers loss. The cumulative AUC is a weighted sum of the performance of the algorithm on the next unseen training example and hence it is a good estimation to the performance of the algorithm on unseen data during training.

The following theorem states a competitive bound. It compares the cumulative AUC of the weight vectors series, $\{ \vtheta_1,\ldots,\vtheta_{T} \}$, resulted from the online algorithm to the best fixed weight vector, $\vtheta^\star$, chosen in hindsight, and essentially proves that, for any sequence of examples, our algorithms cannot do much worse than the best fixed weight vector. Formally, it shows that the cumulative area \emph{above} the curve, $1-\widehat{\text{AUC}}$, is smaller than the weighted average loss $\ell(\vx^{a_t,+}_t,\vx^{a_t,-}_t,a_t, \vtheta_{t-1})$ of the best fixed weight vector $\vtheta^\star$ and its weighted complexity. That is, the cumulative AUC of the iterative training algorithm is going to be high, given that the loss of the best solution is small, the complexity of the best solution is small and that there are reasonable number of rounds, $T$.

\begin{theorem} \label{thm:online}
  Let $\{(\vx^{a_t,+}_t,\vx^{a_t,-}_t,a_t)\}^{T}_{t=1}$ be a
  set of training examples and assume that we have
  $\| \vvarphi(\vx,a) \| \le 1/\sqrt{2}$ for all $\vx$ and $a$. Let $\vtheta^\star$ be the best weight vector selected under some optimization criterion by observing all instances in hindsight. Then,
\begin{equation} 
1 - \widehat{\text{AUC}} \leq \frac{\lambda}{T}\|\vtheta^\star\|^2 + \frac{2}{T}
\sum_{t=1}^T \ell(\vx^{a_t,+}_t,\vx^{a_t,-}_t,a_t, \vtheta^\star) .
\end{equation}
where $\lambda \le 1$ and $\widehat{\text{AUC}}$ is the cumulative AUC defined in \eqref{eq:cumulative_auc}.
\end{theorem}

\begin{proof}
Denote by $\ell_t(\vtheta)$ the instantaneous loss the weight vector $\vtheta$ suffers on the $t$-th round, that is, $\ell_t(\vtheta) = \ell(\vx^{a_t,+}_t,\vx^{a_t,-}_t,a_t, \vtheta_{t-1})$. 
The proof of the theorem relies on Lemma 1 and Theorem 4 in \cite{crammer2006online}. Lemma 1 in \cite{crammer2006online} implies that,
\begin{equation}\label{eq:lemma1}
\sum_{t=1}^T \tau_t \Big( 2\ell_i(\vtheta_{t-1}) 
- \tau_t \| \vvarphi(\vx^{a_t,+},a_t) - \vvarphi(\vx^{a_t,-},\ah^-_t)  \|^2 - 2\ell_t(\vtheta^\star) \Big) \leq \|\vtheta^\star\|^2.
\end{equation}

Now if the algorithm makes a prediction mistake and the predicted confidence of app $\ah^-$ in $\vx^{a_t,-}_t$ is higher than the confidence of the app $a_t$ in $\vx^{a_t,+}_T$ then $\ell_t(\vtheta_{t-1}) \geq 1$. Using the assumption that $\| \vvarphi(\vx,a) \| \le 1/\sqrt{2}$, which in turn means that 
$$
\| \vvarphi(\vx^{a_t,+},a_t) - \vvarphi(\vx^{a_t,-},\ah^-_t) \|^2 \leq 1,
$$
and the definition of $\tau_t$ given in \eqref{eq:tau}, we conclude that if a prediction mistake occurs then it holds that
\begin{align} \nonumber
\tau_t \ell_t(\vtheta_{t-1}) &\ge  \min\left\{\frac{ \ell_t(\vtheta_{t-1}) }{ \| \vvarphi(\vx^{a_t,+},a_t) - \vvarphi(\vx^{a_t,-},\ah^-_t) \|^2 } , \frac{1}{\lambda}\right\} \\  \nonumber
&\ge  \min\left\{ 1 , \frac{1}{\lambda}\right\} = 1.
\end{align}

Summing over all the prediction mistakes made on the entire set of examples and taking into account that $\tau_t \ell_t(\vtheta_{t-1})$ is always non-negative, we have
\begin{equation}\label{eq:ineq1}
\sum_{t=1}^T \tau_t \ell_t(\vtheta_{t-1}) 
\ge 
\sum_{t=1}^T  \I\Big[\vtheta_{t-1}\cdot\vvarphi(\vx^{a_t,+}_t,a_t) \le \vtheta_{t-1}\cdot\vvarphi(\vx^{a_t,-}_t, a_t) \Big].
\end{equation}

Again using the definition of $\tau_t$, we know that $\tau_t \ell_t(\vtheta^\star) \le  \ell_t(\vtheta^\star)/\lambda$ and that $\tau_t \|\Delta\vvarphi_t\|^2 \le \ell_t(\vtheta_{t-1})$. Plugging these two inequalities and \eqref{eq:ineq1} into \eqref{eq:lemma1} we get
\begin{equation}
\sum_{t=1}^T  \I\Big[\vtheta_{t-1}\cdot\vvarphi(\vx^{a_t,+}_t,a_t) \le \vtheta_{t-1}\cdot\vvarphi(\vx^{a_t,-}_t, a_t) \Big] 
\le \lambda \|\vtheta^\star\|^2 + 2\sum_{t=1}^T \ell_t(\vtheta^\star) .
\end{equation}
The theorem follows by replacing the sum over prediction mistakes to a sum over prediction hits and plugging the definition of the cumulative AUC given in \eqref{eq:cumulative_auc}. 
\end{proof}


\section{Feature Primitives and Feature Functions}\label{sec:features}

All our features are based on basic primitives that are measured from the device. The feature primitives denoted $\vx\in\Xc\subset \R^n$. They are the device time and date, the device location, and hardware related signals (e.g., headphones connected indication, current Wi-Fi network, Bluetooth devices, etc.). 

On the top of the feature primitives, we design a set of feature functions which allow us to incorporate into the feature design the presumed app. The idea is that each feature function takes as input a vector of feature primitives $\vx$, which describes the context information, and an app $a$, and returns a scalar which should be related to whether the app $a$ corresponds to the the user's preference at context $\vx$. The feature functions map the context vector $\vx$ along with a proposed app $a$ to a vector of fixed dimension in $\R^d$.

Our feature functions have two sets of representations: \emph{contextual features} and \emph{app-dependent features}. Before turning to describe the actual feature, we define them formally. The set of contextual features is composed of a set of non-linear functions of the feature primitives $\vx$ and denoted as $\vpsi(\vx)$, where $\vpsi:\Xc\to\R^{d_{\psi}}$, where $d_{\psi}$ is the number of contextual features. The set of app-dependent features is a non-linear functions of the feature primitives $\vx$ and the app $a$ and denoted $\vphi(\vx,a)$, where $\vphi:\Xc\times\Ac\to\R^{d_{\phi}}$, where $d_{\phi}$ is the number of app-dependent features. 

The new score function is of the form:
\begin{equation}\label{eq:linear_score_fucntion2}
 f_{\vtheta}(\vx, a) = \vw\cdot\vphi(\vx,a) + \vw^a\cdot\vpsi(\vx),
\end{equation}
where $\vw\in\R^{d_{\phi}}$ and $ \vw^a \in \R^{d_{\psi}}$ for all $a\in\Ac$ are the weight vectors that replace $\vtheta$. While this function looks somewhat different from the score in \eq{eq:linear_score_fucntion}, it is straightforward to show the form in \eq{eq:linear_score_fucntion} is more general. Moreover, the context features could all have been represented as app-dependent features. However, in our setting it is more efficient and convenient to keep both set of feature representations. This allows us to manage memory of the context features in a sparse way.

\subsection{Contextual Features}

Contextual features represent the state of the device, and they are expressed as a set of functions over the feature primitives, such as day-of-week,  headphones connected, and known-location are used and describe in detail below. While the contextual feature do not depend on any specific app, they are weighed from each app separately. It means that we expect to have a high weight for when the feature \emph{day-of-week} equals \texttt{Saturday} for the app \texttt{BestParking}. 

We describe now the context feature vector $\vpsi(\vx)$ that are common to all apps. The first set of features is \emph{time-based} features and includes the following features
\begin{enumerate}
\item hour-of-day: $\{0, 1, \ldots, 23\}$
\item day-of-week: $\{Mon, Tue, \ldots, Sun\}$
\item part-of-day: $\{dawn, morning, noon, afternoon, \ldots \}$
\item weekend: $\{yes, no\}$, where the weekend days are country-specific
\end{enumerate}
The second set is \emph{location-based} features. These features are based in the notion of \emph{known location}. Known location is defined as a recurring location, that was visited within 50 meters in the last month.
\begin{enumerate}
\setcounter{enumi}{4}
\item ID of the known location
\item is the current location known: $\{yes, no\}$
\item just entered a known location: defined by a decay scoring function $10^{-t/15}$, where $t$ is the time since entering the location in minutes.
\item just left known: same function, but $t$ is the time of a transition into unknown locations.
\end{enumerate}  
The last set of context feature is \emph{hardware-based} features.
\begin{enumerate}
\setcounter{enumi}{8}
\item headphones connected: $\{yes, no\}$
\item headphones just connected: defined using the decay function above
\item headphones just disconnected: defined using the decay function above
\item Wi-Fi network connected: $\{yes, no\}$
\item Wi-Fi network SSID 
\item Wi-Fi just connected: defined using the decay function above
\item Wi-Fi just disconnected: defined using the decay function above
\item Bluetooth network connected: $\{yes, no\}$
\item Bluetooth network SSID 
\item Bluetooth just connected: defined using the decay function above
\item Bluetooth just disconnected: defined using the decay function above
\end{enumerate}  
The last feature models the dependency of the current app given the most 5 frecent apps. 
\begin{enumerate}
\setcounter{enumi}{19}
\item The 5 scores corresponds to the 5 most frecent apps according to the frecency predictor in \eq{eq:frecency}, where $T=1$ hour.
\end{enumerate}  
  
\subsection{App-dependent Features}

The second representation of features includes non-linear function functions that are computed for each app separately. We have three feature function.

Let $T_a$ be the set of set of time-samples that the app $a$ was used. Those are absolute time values. The first set feature functions scale the current time $t$ relative to the previous time-stamps of the app $a$:
$$
\vphi_{1,h}(\vx,a) = \frac{1}{|T_{a}|}\sum_{t_s \in T_{a}} 0.5 \left( 1+e^{-(t-t_s)^2/2h^2} \right)
$$
for $h$ in $\{1, 1.5, 3\}$ days. This feature function gives high score to recent apps for which $t-t_s$ is small in the order of $h$.

The next set of feature functions is similar, but it refers to the relative times within a day. That is, if an app $a$ is clicked every day around 9:30am, the score would be high if it is used around the same time again,
$$
\vphi_{2,h}(\vx,a) = \frac{1}{|T_{a}|}\sum_{t_s \in T_{a}} e^{-\Delta_{\textrm{SecOfDay}}(t,t_s)^2/2h^2}
$$
for $h$ in $\{60, 600, 1500\}$ seconds. The function $\Delta_{\textrm{SecOfDay}}(t,t_s)$ is the difference between absolute times $t$ and $t_s$, translated to seconds of day.

The third set of feature functions score the apps based on the locations it is used. Let $l$ be the current location in terms of latitude-longitude. Let $L_{a}$ be the set of the locations the app $a$ was used (latitude-longitude). Let $\Delta_{\textrm{LatLong}}(l,l_s)$ be the distance in meters between the location $l$ and the location $l_s$, then,
$$
\vphi_{3,h}(\vx,a) = \frac{1}{|L_{a}|}\sum_{l_s \in L_{a}} e^{\Delta_{\textrm{LatLong}}(l,l_s)^2/2h^2}
$$
for $h$ in $\{50, 200, 1000\}$ meters.
\section{Experimental Results}
\label{sec:experiments}


In this section we present the performance of our algorithm against other algorithms with different evaluation metrics. We start by comparing the accuracy of the proposed algorithm to other algorithms and with different evaluation metrics. Then, we verify our results with different number of users on different time periods. We continue with an experiment that check the accuracy and the convergence of the online algorithm over time. We conclude with results on the influence of the app ranking on its accuracy.

The data was collected from a set of 1000 randomly selected users which were active users of EverythingMe Launcher for 180 days. This set containing 5,181,312 app-click entries.

On each of these data-sets, we ran the three prediction algorithms - $k$MFU (using $k=4$), Frecency (using $p=0.1, T=60 days$) and AUC-PA (using $C=0.02$).

\subsection{App prediction performance}

We start by comparing the performance of the online algorithm with the $k$MFU algorithm and the Frecency algorithm. We randomly selected 1000 devices and extracted their data for time span of 180 days. Specifically we extracted the context information at times which correspond to app-click events. Then, for each such event, we extract a vector of feature primitives and feature functions and predict a set of 4 apps using $k$MFU, Frecency and AUC-PA. We compared the predictions of the algorithms with the actual app that was clicked. The performance results are given in \fig{fig:algo_performance}. We present 3 evaluation metrics: (a) precision -- for each device we count the times the predicted app is also the clicked app, and then average over all devices ; (b) per-app precision -- for each device and for each app we checked the precision, then averaged over all apps and devices; (c) AUC -- for each device we computed the AUC and averaged over devices. Note that AUC cannot be computed to $k$MFU, since it constantly predicts the same $k$ apps and only them.

\begin{figure*}[th]
\centering
\subfloat[Precision]{%
  \includegraphics[width=0.3\linewidth]{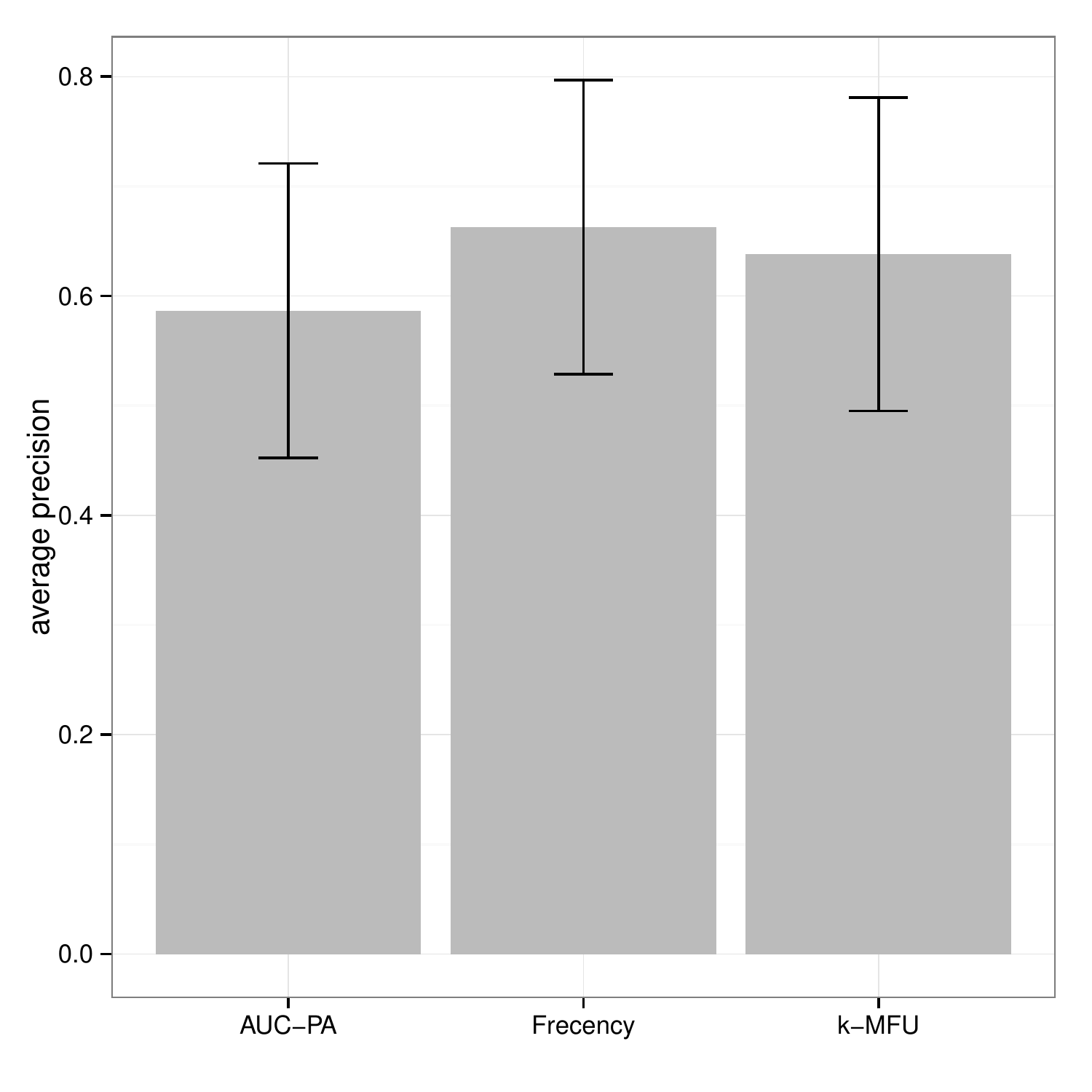}%
  \label{fig:compare-acc}%
}
\subfloat[Precision per-app]{%
  \includegraphics[width=0.3\linewidth]{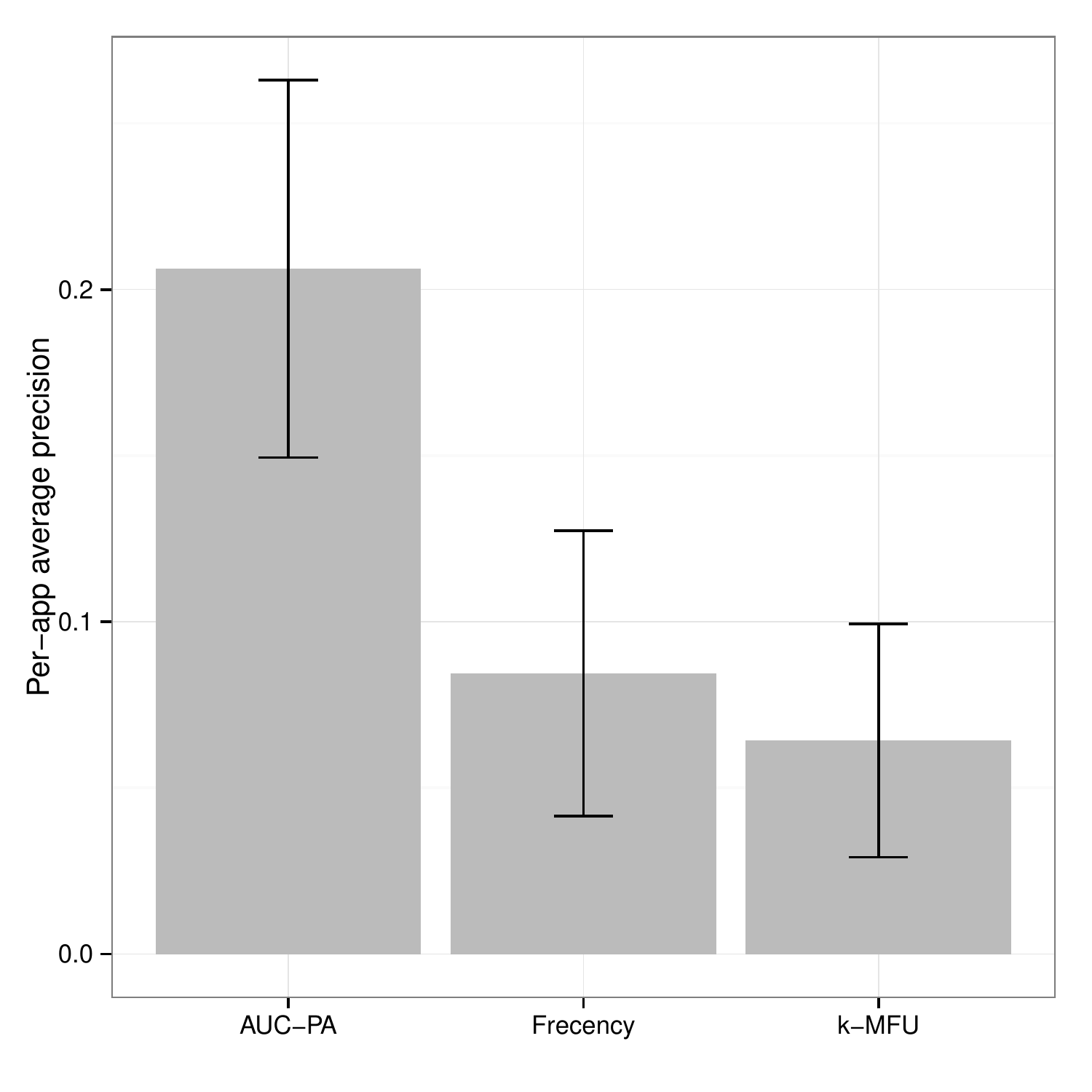}%
  \label{fig:compare-per-app-acc}%
}
\subfloat[AUC]{%
  \includegraphics[width=0.3\linewidth]{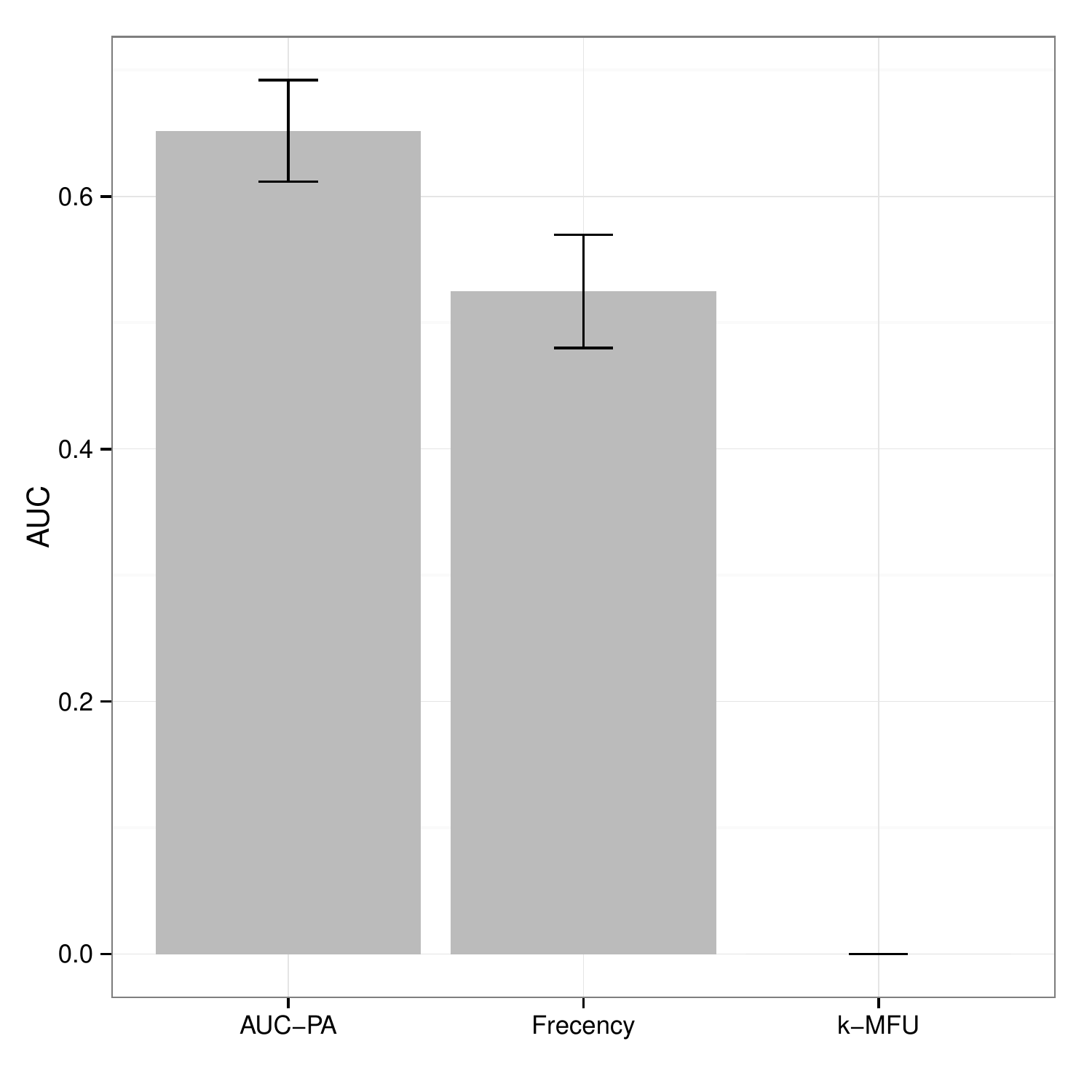}%
  \label{fig:compare-auc}%
}
\caption{Performance of different algorithms: $k$MFU with green circles, Frecency with red x-marks and AUC-PA with blue crosses. The performance are evaluated using (a) precision, (b) per-app precision, and (c) AUC (note that AUC cannot be computed to $k$MFU, since it always predict the same 4 apps).}
\label{fig:algo_performance}
\end{figure*}

Recall that the $k$MFU algorithm always predicts the same $k$ apps --- the most frequent ones.
As expected it has the best precision compared to Freceny and AUC-PA. On the other hand when comparing the per-app precision and the AUC, the AUC-PA outperforms $k$MFU and Frecency.

\subsection{Prediction Quality Over Time}

In the next experiment we analyze the quality of the prediction over the entire time period. We use the same randmoly selected group of 1000 devices, and calculate the (a) per-app precision and (b) AUC in a sliding window of one week. That is, each point in the following graph represents a time period of one week starting 3 days beforehand and ending 3 days afterwords.

We consider the performance of the algorithms over since it is installed on the device and on, that is we would like to understand the behavior over time and how it converges over time. This is very important issue, mainly in order to understand how fast prediction would be relevant for new users (the cold start problem).

\begin{figure*}[th]
\centering
\subfloat[Per-app precision]{%
  \includegraphics[height=6cm]{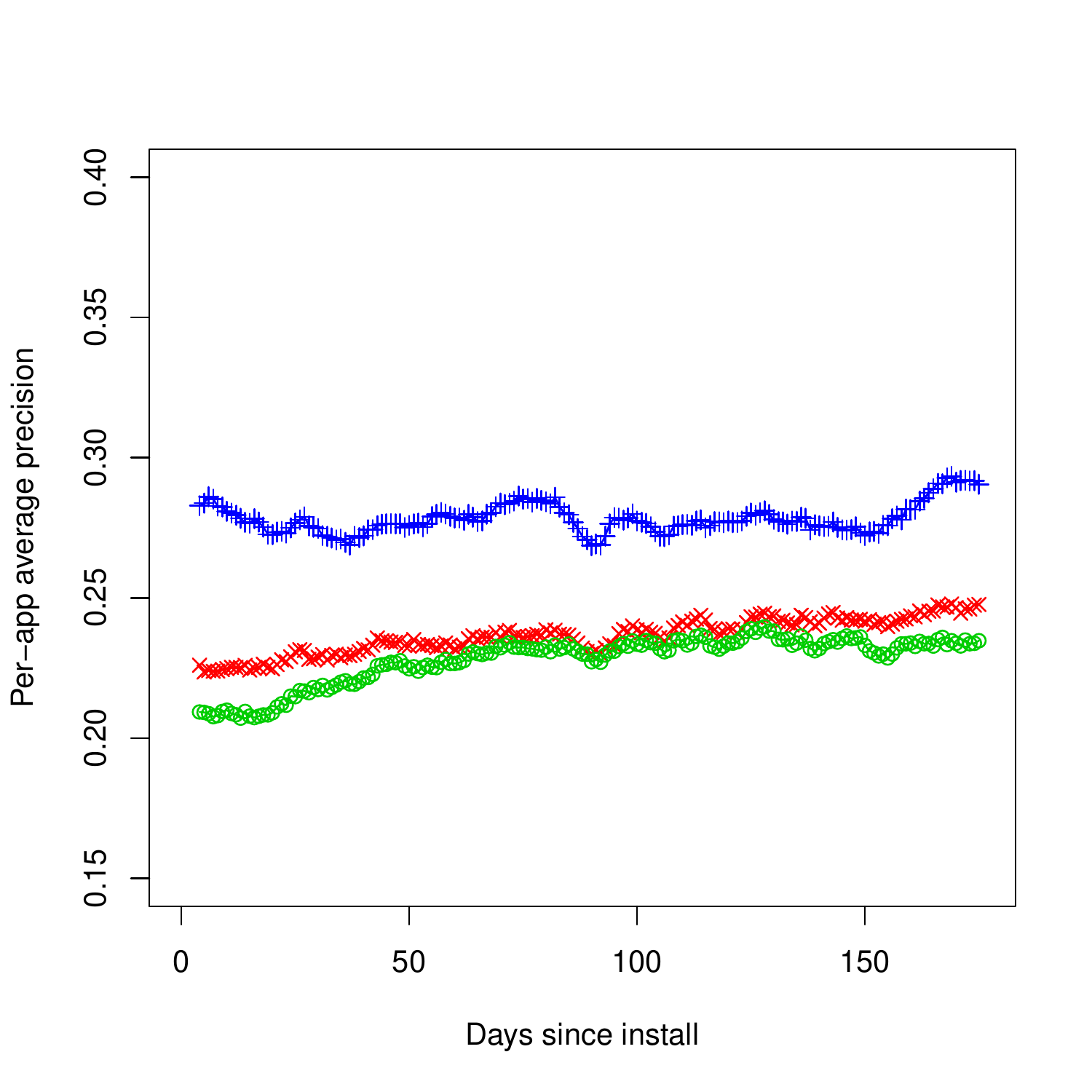}%
  \label{fig:per_age_acc1}%
}
\subfloat[AUC]{%
    \includegraphics[height=6cm]{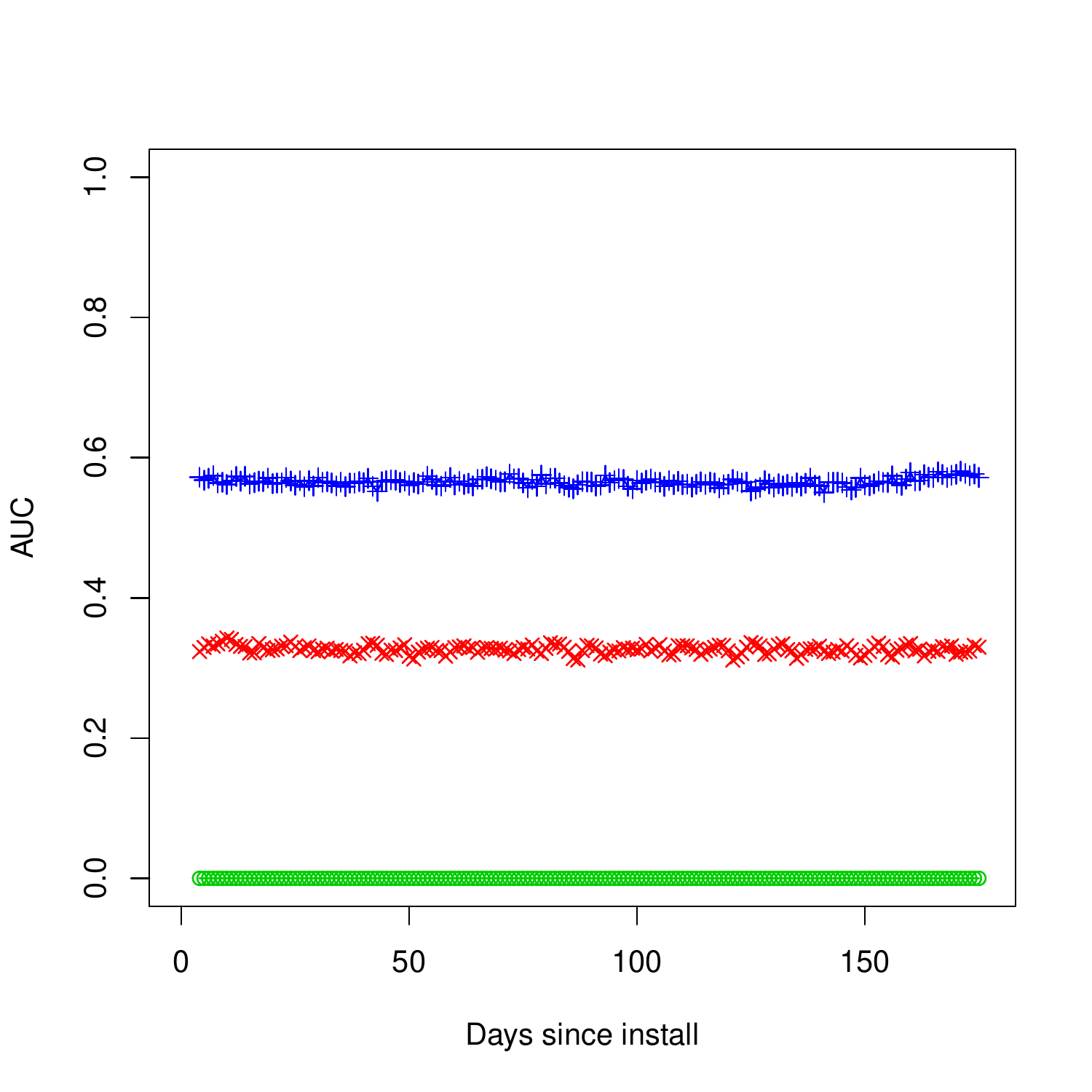}%
    \label{fig:per_age_auc2}%
}
\caption{Performance of different algorithms: $k$MFU with green circles, Frecency with red x-marks and AUC-PA with blue crosses. Measured by (a) per-app average precision and (b) AUC. Each graph shows the local performance of the 3 algorithms (with a sliding window of one week)
}
\label{fig:per_age2}
\end{figure*}

We can see that all algorithms reach stable performance in the first couple of days, and maintain the same performance throughout the entire time period. AUC-PA outperforms Frecency and $k$MFU consistently both in the AUC as well as with the per-app precision.

Something that's worth noting is the fact that as time passes, the prediction algorithms need to take into consideration more apps. According to the test data, the average amount of distinct apps used by a user accumulated over 30 days is 30.23, after 90 days is 57.92 and after 180 days, 76.14. A simple explanation for this increase is that most apps have a limited lifetime on user's devices. Some apps are installed, used for some time and then forgotten or uninstalled. Other apps are simply used very rarely.

Both $k$MFU and Frecency ignore the lesser-used apps by design, as both favour the more used apps. It is therefore interesting to see that in spite the increase in the amount of possible predictions AUC-PA keeps a stable per-app precision rate. This demonstrates the ability of the prediction algorithm to adapt to a moving hypothesis.




\subsection{Performance versus the usage-rank}

This final experiment demonstrates the ability of AUC-PA to predict from the long-tail of lesser-used apps. As previously stated, the average user will use on average 17.6 2nd-tier apps during a single month (i.e. ones which are not positioned on the App Dock or the Home Screen). We therefore assign for each app and device a usage-rank - 0 for the most used app, 1 for the second most used app and so on. Next, we average on all devices the prediction precision of all of the $n$-ranked apps.

\begin{figure}[th]
\centering
  \includegraphics[height=6cm]{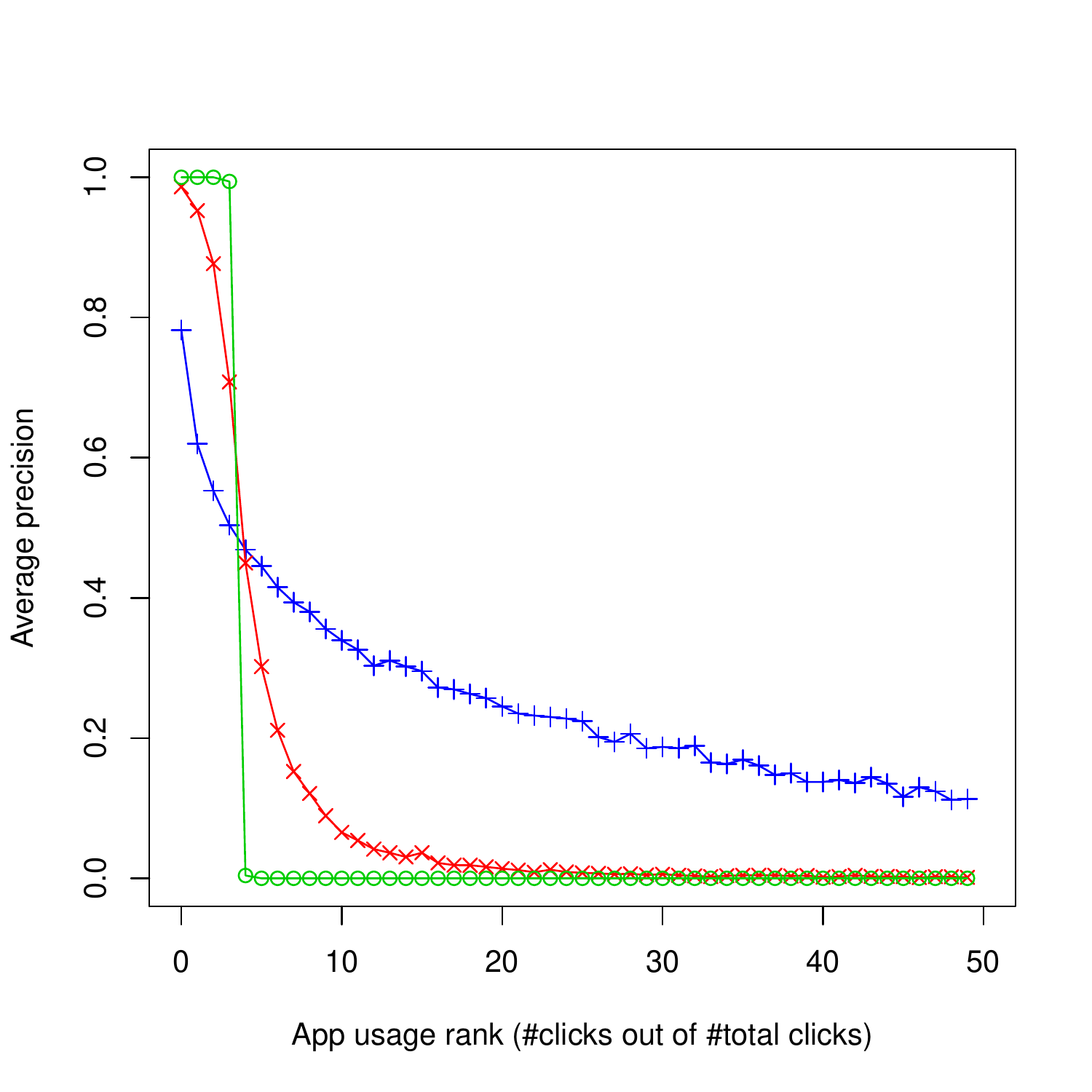}%
  \label{fig:per_age_acc3}%
\caption{Performance of different algorithms: $k$MFU with green circles, Frecency with red x-marks and AUC-PA with blue crosses. Measuring average precision vs. the usage-rank of apps.
}
\label{fig:per_age}
\end{figure}

As expected, $k$MFU performs very well on ranks 0-3 and very badly on the rest. This is a direct result of the fact that it always returns the $k$ most used apps, $k=4$ in our setting. Frecency performs slightly worse for the most used apps and slightly better for the lesser ranked apps, although by the time we reach the 10th place performace is quite poor. AUC-PA shows significantly better results for the apps with lower ranks, while taking a hit in the top-ranked apps.
\section{Discussion}
\label{sec:discussion}

In this paper we presented an algorithm for predicting the set of apps the user will probably use, based on the device's contextual information (time, location, etc.). The algorithm was designed to optimize the user's personalization and to promote the prediction of less used apps in an appropriate context by aiming at the maximization of AUC, rather than simply the raw clicks. The algorithm runs efficiently on the device in an online fashion and constantly updates its hypothesis to handle changes in user's preference over time. 

In a set of experiments we showed that the algorithm attains high performance when evaluated using AUC or when evaluated using the normalized precision (frequently and rarely used apps are weighted the same). We also showed that the algorithm converges, on the average, within few days. Last we showed that the algorithm attains a high prediction rate for those apps which are ranked below the most used apps.

Future work is focused on finding new features to improve performance in general. More specifically, we are interested in including features that indicate the user is located in a special point-of-interest, like a school, an airport, or a shopping mole. Such features, for example, will allow the algorithm to support a different set of app preferences at those location. Another possible direction which we are investigating now is the option to suggest the user to install a new app that would fits its preference in a given context. Another interesting direction that we will explore in future work is a sequential modeling of app usage.
\bibliographystyle{plain}
\bibliography{references}
\end{document}